\documentclass[11pt]{myclass}
\usepackage{euscript}

\usepackage{amssymb}
\usepackage{epsfig}
\usepackage{xspace}
\usepackage{color}
\usepackage{bbm}

\usepackage{wrapfig}
\usepackage{algorithm}
\usepackage{algorithmic}

\newcommand{\eps}{\varepsilon}
\newcommand{\disc}{\mathsf{disc}}

\DeclareMathOperator{\kde}{\textsc{kde}}
\DeclareMathOperator{\poly}{\mathsf{poly}}

\newcommand{\dir}[1]{\ensuremath{\mathsf{d}#1}}
\newcommand{\one}{\mathbbm{1}}

\newcommand{\km}{\ensuremath{\hat{\mu}}}
\renewcommand{\H}{\ensuremath{{\Eu{H}_K}}}
\newcommand{\Rc}{\ensuremath{{\mathbbm{1}}_R}}

\newcommand{\Eu}[1]{\ensuremath{\EuScript{#1}}}
\newcommand{\R}{\ensuremath{\mathbb{R}}}

\newcommand{\norm}[1]{\left\lVert#1\right\rVert}
\newcommand{\normr}[1]{\left\lVert#1\right\rVert_{\Eu{H}_K}}
\newcommand\numberthis{\addtocounter{equation}{1}\tag{\theequation}}

\newcommand{\abs}[1]{\left| #1 \right|}

\title{Improved Coresets for Kernel Density Estimates}
\author{Jeff M. Phillips\thanks{Thanks to supported by NSF CCF-1350888, IIS-1251019, ACI-1443046, CNS-1514520, and CNS-1564287.} 
     \\University of Utah 
     \and WaiMing Tai
	\\ University of Utah}

\begin{document}

\begin{titlepage}
\maketitle

\begin{abstract}
We study the construction of coresets for kernel density estimates.  That is we show how to approximate the kernel density estimate described by a large point set with another kernel density estimate with a much smaller point set.  For characteristic kernels (including Gaussian and Laplace kernels), our approximation preserves the $L_\infty$ error between kernel density estimates within error $\eps$, with coreset size $2/\eps^2$, but no other aspects of the data, including the dimension, the diameter of the point set, or the bandwidth of the kernel common to other approximations.  When the dimension is unrestricted, we show this bound is tight for these kernels as well as a much broader set.  

This work provides a careful analysis of the iterative Frank-Wolfe algorithm adapted to this context, an algorithm called \emph{kernel herding}.  This analysis unites a broad line of work that spans statistics, machine learning, and geometry.  

When the dimension $d$ is constant, we demonstrate much tighter bounds on the size of the coreset specifically for Gaussian kernels, showing that it is bounded by the size of the coreset for axis-aligned rectangles.  Currently the best known constructive bound is $O(\frac{1}{\eps} \log^d \frac{1}{\eps})$, and non-constructively, this can be improved by $\sqrt{\log \frac{1}{\eps}}$.  This improves the best constant dimension bounds polynomially for $d \geq 3$.  
\end{abstract}
\end{titlepage}

\section{Introduction}

A kernel density estimate~\cite{Par62} of a point set $P \subset \mathbb{R}^d$ smooths out the point set to create a continuous function $\kde_P : \mathbb{R}^d \to \mathbb{R}$.  
This object has a rich history and many applications in statistical data analysis~\cite{Sil86,DG84,Sco92}, with many results around the question of if $P$ is drawn iid from an unknown distribution $\psi$, how well can $\kde_P$ converge to $\psi$ as a function of $|P|$ (mainly in the $L_2$~\cite{Sil86,Sco92} and $L_1$~\cite{DG84} sense).  

Then kernel techniques in machine learning~\cite{SS02} developed the connection of kernel density estimates to reproducing kernel Hilbert spaces (RKHS), which are infinite dimensional function spaces (each $\kde_P$ is a point in such a space).  From these techniques grew much of non-linear data analysis (e.g., kernel PCA, kernel SVM).  In particular, an object in the RKHS called the \emph{kernel mean} is another representation of $\kde_P$, and its sparse approximation plays a critical role in distribution hypothesis testing~\cite{GBRSS12,HBCM13}, Markov random fields~\cite{CWS10}, and even political data analysis~\cite{FWS15}.  Through a simple argument (described below), the standard approximation of the kernel mean in the RKHS implies a $L_\infty$ approximation bound of the kernel density estimate in $\mathbb{R}^d$~\cite{CWS10,SZSGS08} (which is stronger than the $L_1$ and $L_2$ variants~\cite{ZP15}).

More recently, the sparse approximation of a kernel density estimate has gained interest from the computational geometry community for its connections in topological data analysis~\cite{PWZ15,FLRWBS14}, coresets~\cite{Phi13}, and discrepancy theory~\cite{harvey2014near}.

In this paper, we provide strong connections between all of these storylines, and in particular provide a simpler analysis of the common sparse kernel mean approximation techniques with application to the strong $L_\infty$-error coresets of kernel density estimates.  With unrestricted dimensions, we show our bounds for KDEs are tight, and in constant dimensions of at least $3$, we polynomially improve the best known bounds so they are now tight up to poly-log factors.    

\paragraph{Formal definitions.}
For a point set $P \subset \mathbb{R}^d$ of size $n$ and a kernel $K : \mathbb{R}^d \times \mathbb{R}^d \to \mathbb{R}$, a \emph{kernel density estimate} $\kde_P$ at $x \in \mathbb{R}^d$ is defined 
$
 \kde_P(x) = \frac{1}{|P|}\sum_{p \in P} K(x,p).
$
Our goal is to construct a subset $Q \subset P$, and bound its size, so that its $\kde$ has $\eps$-bounded $L_\infty$ error:
\[
\| \kde_P - \kde_Q \|_\infty = \max_{x \in \mathbb{R}^d} \left| \kde_P(x) - \kde_Q(x) \right| \leq \eps.
\]
We call such a subset $Q$ an \emph{$\eps$-coreset of a kernel range space $(P,\Eu{K})$} (or just an \emph{$\eps$-kernel coreset} for short), where $\Eu{K}$ is the set of all functions $K(x,\cdot)$ represented by a fixed kernel $K$ and an arbitrary center point $x \in \mathbb{R}^d$.

While there is not one standard definition of a kernel, many of these kernels have properties that unite them.  
Common examples are 
the Gaussian kernel $K(x,p) = \exp(-\|x-p\|^2/\sigma^2)$, 
the Laplace kernel $K(x,p) = \exp(-\|x-p\|/\sigma)$, 
the ball kernel $K(x,p) = \{1$ if  $\|x-p\|/\sigma \leq 1$; and $0$ otherwise\}, 
and 
the triangle kernel $K(x,p) = \max \{0, 1 - \|x-p\|/\sigma\}$.  
The parameter $\sigma$ is often called the \emph{bandwidth} and controls the level of smoothing.  
All of these kernels (and indeed most) are \emph{shift invariant}, thus can be written with a single input $z = \|x-p\|$ as a $f(z = \|x-p\|) = K(x,p)$.  We have chosen to normalize all kernels so $f(0) = K(x,x) = 1$.

The \emph{kernel distance}~\cite{HB05,glaunesthesis,
JoshiKommarajuPhillips2011,PhillipsVenkatasubramanian2011} (also called \emph{current distance} or \emph{maximum mean discrepancy}) is a metric~\cite{muller1997integral,SGFSL10} between two point sets $P$, $Q$ (as long as the kernel used is characteristic~\cite{SGFSL10}, a slight restriction of being positive definite~\cite{aronszajn1950theory,Wah99}, this includes the Gaussian and Laplace kernels).
Define the similarity between the two point sets as 
$
\kappa(P,Q) = \frac{1}{|P|}\frac{1}{|Q|} \sum_{p \in P} \sum_{q \in Q} K(p,q),
$ 
and the kernel distance as 
$
D_K(P,Q) = \sqrt{\kappa(P,P) + \kappa(Q,Q) - 2 \kappa(P,Q)}.
$
When $Q$ is a single point $x$, then $\kappa(P,x) = \kde_P(x)$.  

If $K$ is positive definite, it is said to have the reproducing property~\cite{aronszajn1950theory,Wah99}.
This implies that $K(p,x)$ is an inner product in a reproducing kernel Hilbert space (RKHS) $\H$.  Specifically, there exists a lifting map $\phi : \mathbb{R}^d \to \H$ so $K(p,x) = \langle \phi(p), \phi(x) \rangle_{\H}$, and moreover the entire set $P$ can be represented as $\Phi(P) = \sum_{p \in P} \phi(p)$, which is a single element of $\H$ and has norm $\|\Phi(P)\|_{\H} = \sqrt{\kappa(P,P)}$.  A single point $x \in \mathbb{R}^d$ also has a norm $\|\phi(x)\|_{\H} = \sqrt{K(x,x)} = 1$ in this space.  
A \emph{kernel mean} of a point set $P$ and a reproducing kernel $K$ is defined 
\[
\km_P = \frac{1}{|P|} \sum_{p \in P} \phi(p) = \Phi(P)/|P| \in \H.
\]
Note that $\|\km_P\|_{\H} \leq K(x,x)$ so in our setting $\|\km_P\|_{\H} \leq 1$.  
Also $D_K(P,Q) = \|\km_P - \km_Q\|_\H$.  

\vspace{-.1in}
\paragraph{Relationship between kernel mean and $\eps$-kernel coresets.}
It is possible to convert between bounds on the subset size required for the kernel mean and an $\eps$-kernel coreset of an associated kernel range space.  But they are not symmetric.  

The Koksma-Hlawka inequality (in the context of reproducing kernels~\cite{CWS10,SZSGS08} when $K(x,x) = 1$) states that
\[
\|\kde_P - \kde_Q\|_\infty \leq  \|\km_P - \km_Q\|_{\Eu{H}_K}.
\]
Since $\kde_P(x) = \kappa(\hat{P}, x) = \langle \km_P, \phi(x) \rangle_\H$ and via Cauchy-Schwartz, for any $x \in \mathbb{R}^d$  
\[
| \kde_P(x) - \kde_Q(x) |
=
| \langle \km_P , \phi(x) \rangle_\H - \langle \km_Q, \phi(x) \rangle_\H |
= 
| \langle \km_P - \km_Q, \phi(x) \rangle_\H |
\leq 
\| \km_P - \km_Q \|_\H.  
\]
Thus to bound $\max_{x \in \mathbb{R}^d} |\kde_P(x) - \kde_Q(x)| \leq \eps$ it is sufficient to bound $\|\km_P - \km_Q\|_\H \leq \eps$.  

On the other hand, if we have a bound $\max_{x \in \mathbb{R}^d} |\kde_P(x) - \kde_Q(x)| \leq \eps$, then we can only argue that $\|\km_P - \km_Q\|_\H \leq \sqrt{2 \eps}$.  
We observe that
\begin{align*}
\|\km_P - \km_Q\|_\H^2 
&= 
D_K(P,Q)^2 = \kappa(P,P) + \kappa(Q,Q) - 2 \kappa(P,Q)
\\ & =
\frac{1}{|P|} \sum_{p \in P} \kde_P(p) + \frac{1}{|Q|} \sum_{q \in Q} \kde_Q(q)  - \frac{1}{|P|} \sum_{p \in P} \kde_Q(p)  - \frac{1}{|Q|} \sum_{q \in Q} \kde_P(q)
\\ &=
\frac{1}{|P|} \sum_{p \in P} (\kde_P(p) - \kde_Q(p)) + \frac{1}{|Q|} \sum_{q \in Q} (\kde_Q(q) - \kde_P(q))
\\ & \leq
\frac{1}{|P|} \sum_{p \in P} (\eps) + \frac{1}{|Q|} \sum_{q \in Q} (\eps)
= 2\eps.  
\end{align*}
We can also take the only inequality the other direction to get the lower bound.  Taking a square root of both sides leads to the implication.  

Unfortunately, the second reduction does not map the other way; a bound on $\|\km_P - \km_Q\|_H^2$ only ensures an average ($L_2$ error) for $\kde_P$ holds, not the desired stronger $L_\infty$ error.

\subsection{Known Results on KDE Coresets}
In this section we survey known bounds on the size $|Q|$ required for $Q$ to be an $\eps$-kernel coreset of the kernel range space $(P,\Eu{K})$.  
We assume $P \subset \mathbb{R}^d$, it is of size $n$, and $P$ has a diameter 
$\Delta  = (1/\sigma) \max_{p, p' \in P} \|p-p'\|$, 
where $\sigma$ is the bandwidth parameter of the kernel.  
We sometimes allow a $\delta$ probability that the algorithm does not succeed.  
Results are summarized in Table \ref{tbl:compare}.  

\begin{table}[t]
\begin{tabular}{|r|c|c|l|}
\hline
Paper & Coreset Size & Runtime & Restrictions 
\\ \hline 
Joshi \etal~\cite{JoshiKommarajuPhillips2011} & $(1/\eps^2)(d + \log(1/\delta))$ & $|Q|$ samples & centrally symmetric, positive
\\
Fasy \etal~\cite{FLRWBS14} &  $(d/\eps^2) \log(d \Delta / \eps \delta)$ & $|Q|$ samples & ..
\\ 
Gretton \etal~\cite{GBRSS12} & $(1/\eps^4) \log(1/\delta)$ & $|Q|$ samples & characteristic kernels
\\ \hline
Phillips~\cite{Phi13} & $(1/\eps\sigma)^{2d/(d+2)} \log^{d/(d+2)} (1/\eps\sigma)$ & $n/\eps^2$ & $(1/\sigma)$-Lipschitz, \; $d$ is constant
\\
Phillips~\cite{Phi13} & $1/\eps$ & $n \log n$ & $d=1$
\\ \hline
Chen \etal~\cite{CWS10} & $1/(\eps r_P)$ & $n/(\eps r_P)$ & characteristic kernels
\\ 
Bach \etal~\cite{BLO12} & $(1/r_P^2)\log (1/\eps)$ & $n\log(1/\eps)/r_P^2$ & characteristic kernels   
\\
Bach \etal~\cite{BLO12} & $1/\eps^2$ & $n/\eps^2$ & characteristic kernels, weighted
\\
Harvey  and Samadi~\cite{harvey2014near} & $(1/\eps)\sqrt{n}\log^{2.5}(n)$ & $\mathsf{poly}(n,1/\eps,d)$ & characteristic kernels
\\
Cortez and Scott \cite{CS15} & $k_0$ \;\; ($\leq (\Delta/ \eps)^d$) & $nk_0$ & $(1/\sigma)$-Lipschitz; \; $d$ is constant
\\ \hline
\textbf{New Result} &  $1/\eps^2$ &  $n/\eps^2$ & characteristic kernels, \emph{unweighted}
\\
\textbf{New Result} &  $(1/\eps) \log^d \frac{1}{\eps}$ &  $n  + \poly(1/\eps)$ & Gaussian, $d$ is constant
\\
\textbf{New Lower Bound} &  $\Omega(1/\eps^2)$ &  $-$ & \textsc{siss} (e.g., Gaussian); $d = \Omega(1/\eps^2)$  
\\\hline
\end{tabular} 
\vspace{-.1in}
\caption{Asymptotic $\eps$-kernel coreset sizes and runtimes
in terms of $\eps$, $n$, $d$, $r_P$, $\sigma$, $\Delta$.   
\textsc{siss} = Shift-invariant, somewhere-steep (see Section \ref{sec:LB}).    
\label{tbl:compare}}
\vspace{-.1in}
\end{table}

\paragraph{Halving approaches.}
Phillips~\cite{Phi13} showed that kernels with a bounded Lipschitz factor (so $|K(x,p) - K(x,q)| \leq C \|p-q\|$ for some constant $C$, including Gaussian, Laplace, and Triangle kernels which have $C = O(1/\sigma)$, admit coresets of size $O((1/\eps\sigma) \sqrt{\log(1/\eps\sigma)})$ in $\mathbb{R}^2$.  For points in $\mathbb{R}^d$ (for $d>1$) this generalizes to a bound of $O((1/\eps\sigma)^{2d/(d+2)} \log^{d/(d+2)} (1/\eps\sigma))$.  
That paper also observed that for $d=1$, selecting evenly spaced points in the sorted order achieves a coreset of size $O(1/\eps)$.  

\paragraph{Sampling bounds.}
Joshi \etal~\cite{JoshiKommarajuPhillips2011} showed that a random sample of size $O((1/\eps^2)(d + \log(1/\delta)))$ results in an $\eps$-kernel coreset for any centrally symmetric, non-increasing kernel.  This works by reducing to a VC-dimensional~\cite{LLS01} argument with ranges defined by balls.  

Fasy~\etal~\cite{FLRWBS14} 
provide an alternative bound on how random sampling preserves the $L_\infty$ error in the context of statistical topological data analysis.  Their bound can be converted to require size  
$O((d/\eps^2) \log(d \Delta/\eps \delta))$, which can improve upon the bound of Joshi~\cite{JoshiKommarajuPhillips2011} if $K(x,x) > 1$ (otherwise, herein we only consider the case $K(x,x) =1$).  
 
Examining characteristic kernels which induce an RKHS in that function space leads to a simpler bound of $O((1/\eps^4) \log(1/\delta))$ as observed by Gretton \etal~\cite{GBRSS12}.   This shows that after $O((1/\eps')^2 \log(1/\delta))$ samples $Q$, then $\|\km_P - \km_Q\|_\H^2 \leq \eps'$.  To bound the non-squared distance by $\eps$ we set $\eps' = \eps^2$, and obtain the bound of $O((1/\eps^4) \log(1/\delta))$.

\paragraph{Iterative approaches.}
Motivated by the task of constructing samples from Markov random fields, Chen \etal~\cite{CWS10} introduced a technique called \emph{kernel herding} suitable for characteristic kernels.  
They showed that iteratively and greedily choosing the point $p \in P$ which when added to $Q$ most decreases the quantity $\|\km_P - \km_Q\|_{\Eu{H}_K}$, will decrease that term at rate $O(r_P/t)$ for $t = |Q|$.  Here $r_P$ is the largest radius of a ball centered at $\km_P \in \H$ which is completely contained in the convex hull of the set $\{\phi(p) \mid p \in P\}$.  
They did not specify the quantity $r_P$ but argued that it is a constant greater than $0$.  

Bach \etal~\cite{BLO12} showed that this algorithm can be interpreted under the Frank-Wolfe framework~\cite{Cla10}.  Moreover, they argue that $r_P$ is not always a constant; in particular when $P$ is infinite (e.g., it represents a continuous distribution) then $r_P$ is arbitrarily large.  
However, when $P$ is finite, they prove that $r_P$ is finite without giving an explicit bound.  
They also makes explicit that after $t$ steps, they achieve $\|\km_P - \km_{Q,w}\|_{\Eu{H}_K} \leq 4 / (r_P \cdot t)$.  They also describe a method which includes ``line search'' to create a weighted coreset $(Q,w)$, so each point $q \in Q$ is associated with a weight $w(q) \in [0,1]$ so $\sum_{q \in Q} w(q) = 1$; then $\km_{Q,w} = \sum_{q \in Q} w(q) \phi(q)$.  For this method they achieve 
$
\|\km_P - \km_{Q,w}\|_{\Eu{H}_K} \leq \sqrt{\exp(-r_P^2 t)}.
$

Bach \etal~\cite{BLO12} also mentions a bound $\|\km_P - \km_{Q,w}\|_{\Eu{H}_K} \leq \sqrt{8 / t}$, that is independent of $r_P$.  It relies on very general bound of Dunn~\cite{dunn1980convergence} which uses line search, or one of Jaggi~\cite{Jag13} which uses a fixed but non-uniform set of weights.  These show linear convergence for any smooth function, including $\|\km_P - \km_{Q,w}\|^2_{\Eu{H}_K}$; taking the square root provides a bound for $\|\km_P - \km_{Q,w}\|_{\Eu{H}_K} \leq \eps$ after $t = O(1/\eps^2)$ steps.  However, the result is a weighted coreset $(Q,w)$ which may be less intuitive or harder to work with in some situations.  

Harvey and Samadi ~\cite{harvey2014near} further revisited kernel herding in the context of a general mean approximation problem in $\mathbb{R}^{d'}$.  That is, consider a set $P'$ of $n$ points in $\mathbb{R}^{d'}$, find a subset $Q' \subset P'$ so that $\|\bar P' - \bar Q'\| \leq \eps$, where $\bar P'$ and $\bar Q'$ are the Euclidean averages of $P'$ and $Q'$, respectively.  This maps to the kernel mean problem with $P' = \{\phi(p) \mid p \in P\}$, and with the only bound of $d'$ as $n$.  
They show that the $r_P$ term can be manipulated by affine scaling, but that in the worst case (after such transformations via John's theorem) it is $O(\sqrt{d'} \log^{2.5} (n))$, and hence 
show one can always set $\eps = O(\sqrt{d'} \log^{2.5} (n) / t) = O((1/t) \sqrt{n} \log^{2.5}(n))$.  We will show that we can always compress $P'$ to another set $P''$ of size $n = O(1/\eps^2)$ (or for instance use the random sampling bound of Joshi \etal~\cite{JoshiKommarajuPhillips2011}, ignoring other factors); then solving for $t$ yields $t = O((1/\eps^2) \log^{2.5}(1/\eps))$.  

Harvey and Samadi also provide a lower bound to show that after $t$ steps, the kernel mean error may be as large as $\Omega(\sqrt{d'}/t)$ when $t = \Theta(n)$.  
This seems to imply (using the $d' = \Omega(n)$ and a $P'$ of size $\Theta(1/\eps^2)$) that we need $t = \Omega(1/\eps^2)$ steps to achieve $\eps$-error for kernel density estimates.  But this would contradict the bound of Phillips~\cite{Phi13}, which for instance shows a coreset of size $O((1/\eps) \sqrt{\log (1/\eps)})$ in $\mathbb{R}^2$.  More specifically, it uses $t = \Theta(d')$ steps to achieve this case, so if $d' = n = \Theta(1/\eps^2)$ then this requires asymptotically as many steps as there are points.  Moreover, a careful analysis of their construction shows that the corresponding points in $\mathbb{R}^d$ (using an inverse projection $\phi^{-1} : \Eu{H}_K \to \mathbb{R}^d$ to a set $P \in \mathbb{R}^d$) would have them so spread out that $\kde_P(x) < c/\sqrt{n}$ (for constant $c$, so $= O(\eps)$ for $n = 1/\eps^2$) for all $x \in \mathbb{R}^d$; hence it is easy to construct a $2/\eps$ size $\eps$-kernel coreset for this point set.

\paragraph{Discretization bounds.}
Another series of bounds comes from the Lipschitz factor of the kernels:  $C = \max_{x,y,z \in \mathbb{R}^d} \frac{K(z,x) - K(z,y)}{\|x-y\|}$.  For most kernels $C$ is small constant.  This implies that 
$\max_{x,y \in \mathbb{R}^d} \frac{\kde_P(x) - \kde_P(y)}{\|x-y\|} \leq C$ for any $P$.  
Thus, we can for instance, lay down an infinite grid $G_{\eps} \subset \mathbb{R}^d$ of points so for all $x \in \mathbb{R}^d$ there exists some $g \in G_\eps$ such that $\|g-x\| \leq \eps$, and  no two $g, g' \in G_\eps$ are $\|g-g'\| \leq 2\eps \sqrt{d}$.  

Then we can map each $p \in P$ to $p_g$ the closest point $g \in G_\eps$ (with multiplicity), resulting in $P_G$.  By the additive property of $\kde$, we know that $\|\kde_P - \kde_{P_G}\|_\infty \leq \eps$.  

Cortes and Scott~\cite{CS15} provide another approach to the sparse kernel mean problem.  They run Gonzalez's algorithms~\cite{Gon85} for $k$-center on the points $P \in \mathbb{R}^d$ (iteratively add points to $Q$, always choosing the furthest point from any in $Q$) and terminate when the furthest distance to the nearest point in $Q$ is $\Theta(\eps)$.  Then they assign weights to $Q$ based on how many points are nearby, similar to in the grid argument above.  They make an ``incoherence'' based argument, specifically showing that $\|\km_P - \km_Q\| \leq \sqrt{1-v_Q}$ where $v_Q = \min_{p \in P} \max_{q \in Q} K(p,q)$.  This does not translate meaningfully in any direct way to any of the parameters we study.  However, we can use the above discretization bound to argue that if $\Delta$ is bounded, then this algorithm must terminate in $O((\Delta/\sigma \eps)^d)$ steps.

\paragraph{Lower bounds.}
Finally, there is a simple lower bound of size $\lceil 1/\eps \rceil - 1$ for an $\eps$-coreset $Q$ for kernel density estimates~\cite{Phi13}.  Consider a point set $P$ of size $1/\eps-1$ where each point is very far away from every other point, then we cannot remove any point otherwise it would create too much error at that location.

\subsection{Our Results}
We have three main results.  
First, in Section \ref{sec:FW}, we study the kernel herding algorithm for characteristic kernels, and show that after $2/\eps^2$ steps (with no other parameters) it creates a subset $Q \subset P$ so that $\|\km_P - \km_Q\|_\H \leq \eps$, and hence $\|\kde_P - \kde_Q\|_\infty \leq \eps$.  Our result is simple and from first principles, and does not required a weighted coreset, unlike Bach \etal~\cite{BLO12}.  

Second, in Section \ref{sec:Gauss-rect}, we prove a new discrepancy reduction, that shows a form of kernel discrepancy for Gaussian kernels is implied by the same coloring with respect to the very commonly studied axis-aligned rectangle range space.  As a result, this allows us to apply a halving approach to create $\eps$-kernel coresets which are of size $O((1/\eps) \log^d \frac{1}{\eps})$ for $d$ constant, polynomially improving upon all bounds for $d \geq 3$.  

Third, in Section \ref{sec:LB}, we show a lower bound, that there exist point sets $P$ in dimension $\Omega(1/\eps^2)$, such that any $\eps$-kernel coreset requires $\Omega(1/\eps^2)$ points.  Because this construction uses $\Omega(1/\eps^2)$ dimensions, it does not contradict the halving-based results.  This applies to every shift-invariant kernel we considered, with a slightly weakened condition for the ball kernel.

\section{Analysis of Frank-Wolfe Algorithm for $\eps$-Kernel Coresets}
\label{sec:FW}

Our first contribution analyzes the Frank-Wolfe algorithm~\cite{FW56} in the context of estimating the kernel mean $\km_P$.  To simplify notation, let $\mu = \km_P$, and for each original point $p_i \in P$ we denote $\phi_i = \phi(p_i) \in \H$ as its representation in $\H$.  Our estimate $\km_Q$ for $\km_P$ will change each step; on the $t$th step it will be labeled $x_t \in \H$.  The algorithm is then summarized in Algorithm \ref{alg:fw}.  

\begin{algorithm}
	\caption{\label{alg:fw} Frank-Wolfe Algorithm}
	\begin{algorithmic} 
		\STATE $x_1\leftarrow \text{any of }\phi_i$
		\FOR {$t=1,2,\ldots,T$}
		\STATE $i_t\leftarrow\arg\min_{i\in\{1,\dots,n\}} \langle x_t-\mu, \phi_i-\mu\rangle_\H$
		\STATE $x_{t+1}\leftarrow \frac{1}{t} \phi_{i_t}+\frac{t-1}{t} x_t$ 
		\ENDFOR
		\STATE \textbf{return} $x_T = \km_Q$. 
	\end{algorithmic}
\end{algorithm}

Before we begin our own analysis specific to approximating the kernel mean, we note that there exists rich analysis of different variants of the Frank-Wolfe algorithm.  In many cases with careful analysis of the eccentricity~\cite{GJ09} or convexity~\cite{BLO12} or dual structure~\cite{Cla10} of the problem, one can attain convergence at a rate of $O(1/t)$ (hiding structural terms in the $O(\cdot)$, where random sampling typically achieves a slower rate of $O(1/\sqrt{t}$).  Recent progress has focused mainly on approximate settings~\cite{FG16} or situations where one can achieve a ``linear'' rate of roughly $O(c^{-t})$~\cite{JL15}.  This faster linear convergence, unless some specific properties of the data existed, would violate our lower bound, and thus is not possible in general.  

\begin{theorem}\label{avg}
After $T \geq 2/\eps^2$ steps, Algorithm \ref{alg:fw} produces $\km_Q$ with $T$ points so $\normr{\km_Q - \km_P} \leq \eps$.
\end{theorem}

\begin{proof}	
	We first obtain the following recursive equation.
	\[
		x_{t+1}-\mu = \frac{1}{t}\phi_{i_t}+\frac{t-1}{t}x_{t-1}-\mu = \frac{1}{t}(\phi_{i_t}-\mu)+\frac{t-1}{t}(x_{t-1}-\mu)
	\]
	Then, by multiplying $t$ to both sides and taking the squared norm, we have
	\begin{align*}
		\normr{t (x_t-\mu)}^2 
		& =
		\normr{(\phi_{i_t}-\mu)+(t-1)(x_{t-1}-\mu)}^2 \\
		& =
		\normr{\phi_{i_t}-\mu}^2+2(t-1)\langle\phi_{i_t}-\mu, x_{t-1}-\mu\rangle_\H + (t-1)^2\normr{x_{t-1}-\mu}^2.  
	\end{align*}
	
Next we use two observations to simplify this.  	
First, since $\mu$ is the mean of $\{x_1, \ldots, x_n\}$ then $0=\sum_{i=1}^n\langle x_t-\mu, \phi_i-\mu\rangle_\H$.  Thus by optimal selection of $i_t$, then $\langle x_t-\mu, \phi_i-\mu\rangle_\H \leq 0$ for all $i\in\{1,\ldots,n\}$.

Second $\norm{\mu}_\H \leq 1$ since it is a convex combination of $\{\phi_1, \ldots, \phi_n\}$, where each $\norm{\phi_i}_\H = 1$.  This implies $\norm{\phi_{i_t}-x_t}_\H\leq \norm{\phi_{i_t}}_\H+\norm{x_t}_\H\leq 2$.  

Applying the above observations,
\[
	t^2\normr{x_t-\mu}^2 \leq (t-1)^2\normr{x_{t-1}-\mu}^2+2.  
\]
Now, by induction on $t$, at step $t=T$, we can conclude that $T^2 \normr{x_T-\mu}^2 \leq 2T$, and hence $\normr{x_T - \mu}^2 \leq 2/T$.  Thus for $T \geq 2/\eps^2$, then $\normr{x_T-\mu} \leq \eps$.  
\end{proof}

\section{Gaussian Kernel Coresets Bounded using Rectangle Discrepancy}
\label{sec:Gauss-rect}

We now present a completely different way to improve the size of kernel coresets, via discrepancy.  In particular we reduce the kernel coreset problem to a rectangle discrepancy problem.  Our result is specific to the Gaussian kernel.  

Let $(P,\Eu{R}_d)$ be the range space with ground set $P \subset \R^d$ and $\Eu{R}_d$ the family of subsets of $P$ defined by inclusion in axis-aligned rectangles.  In particular, for any point $x \in \R^d$, let $x_i$ represent its $i$th coordinate.  
We define a combinatorial rectangle $R \subset P$ as $R = \{x \in P \mid m_i \leq x_i \leq M_i\}$ by $2d$ values on the left $m_1, m_2, \ldots, m_d$ and right $M_1, M_2, \ldots, M_d$.  
Let $\Rc : \R^d \to \{0,1\}$ be the \emph{characteristic function of the rectangle} $R$; for a point $x \in \R^d$ it returns $1$ if $x \in R$ and $0$ otherwise.  

Let $\chi : P \to \{-1, +1\}$ be a coloring on $P$.  
Then the discrepancy of $(P,\Eu{R}_d)$ with respect to a coloring $\chi$ is defined
$
\disc(P,\chi,\Eu{R}_d) = \max_{R \in \Eu{R}_d} | \sum_{p \in R} \chi(p) |.
$

Following Joshi \etal~\cite{JKPV11}, we can define a similar concept for kernels.  Let $\Eu{K}_d$ be the family of functions $\Eu{K}_d = \{K(x,\cdot) \mid x \in \R^d\}$, for a specific kernel $K$ which in this case will be Gaussian.  Now define the kernel discrepancy of $(P, \Eu{K}_d)$ with respect to a coloring $\chi$ as 
$
\disc(P, \chi, \Eu{K}_d) = \max_{x \in \R^d} | \sum_{p \in P} \chi(p) K(x,p)|.
$

\begin{lemma}
\label{lem:KtoR}
For any point set $P \subset \R^d$ and coloring $\chi$, 
we have
$\disc(P,\chi, \Eu{K}_d)\leq \disc(P, \chi, \Eu{R}_d)$.
\end{lemma}

The proof of the above lemma hinges on two key observations.
The first one is the Gaussian kernel has an important property in common with the characteristic function for an axis-aligned rectangle: multiplicative separability.  
Namely, both of them can be expressed as the product of factors, each corresponding to one dimension.  
Another observation is that the Gaussian function could be expressed as the average of characteristic function for a family of intervals.
Combining the above two facts, one can derive that the signed discrepancy for the Gaussian kernel is the average of signed discrepancy for a family of axis paralleled rectangles.
Finally, by triangle inequality, we will complete the proof of lemma.

\begin{proof}
Define that $\one_{[-r,r]}(x) = 1$ iff $r >\abs{x}$, otherwise $\one_{[-r,r]}(x) = 0$.   
For any $x\in \mathbb{R}$, consider the term $\int_0^\infty 2r\exp(-r^2)\one_{[-r,r]}(x) \dir{r}$, which can be expanded as follows.  
\[
\int_0^\infty 2r\exp(-r^2)\one_{[-r,r]}(x) \dir r 
= 
\int_{\abs{x}}^\infty 2r\exp(-r^2) \dir r 
= 
\left. -\exp(-r^2)\right\vert_{\abs{x}}^\infty = \exp(-x^2)
\numberthis\label{intform}
\]

Next we show that we can decompose the Gaussian kernel $K(c,\cdot) \in \Eu{K}_d$, as follows for any $x \in \R^d$.  
\begin{align*}
		K(c,x)
		& =
		\exp(-\norm{x-c}^2) 
		= 
		\exp(- (\sum_{i=1}^d (x_i - c_i)^2)) 
		=
		\prod_{i=1}^d  \exp(-(x_i-c_i)^2) \\
		& =
		\prod_{i=1}^d \left( \int_0^\infty 2r_i\exp(-r_i^2) \one_{[-r_i, r_i]}(x_i-c_i) \dir r_i \right) \qquad\text{from ($\ref{intform}$)}\\
		& =
		\prod_{i=1}^d \left( \int_0^\infty 2r_i\exp(-r_i^2) \one_{[c_i-r_i, c_i+r_i]}(x_i) \dir r_i \right) \\
		& = 
		\int_0^\infty\dots\int_0^\infty \left(\prod_{i=1}^d  2r_i\exp(-r_i^2)  \one_{[c_i-r_i, c_i+r_i]}(x_i) \right)  \dir r_1\dots \dir r_d \\
		& = 
		\int_0^\infty\dots\int_0^\infty \prod_{i=1}^d\left( 2r_i\exp(-r_i^2)\right) \prod_{i=1}^d \left(\one_{[c_i-r_i, c_i+r_i]}(x_i) \right) \dir r_1\dots \dir r_d \\
		& = 
		\int_0^\infty\dots\int_0^\infty \prod_{i=1}^d\left( 2r_i\exp(-r_i^2)\right) \one_{\prod_{i=1}^d[c_i-r_i,c_i+r_i]}(x) \dir r_1\dots \dir r_d 
\end{align*}
The third line is from 
\[
\mathbbm{1}_{[-r_i, r_i]}(x_i-c_i)=1 
\;\; \Leftrightarrow \;\; 
-r_i \leq x_i-c_i \leq r_i 
\;\; \Leftrightarrow \;\; 
c_i-r_i \leq x_i \leq c_i+r_i 
\;\; \Leftrightarrow \;\; 
\mathbbm{1}_{[c_i-r_i, c_i+r_i]}(x_i)=1.
\]
The fourth line is because for any $f,g$, $\left(\int_0^\infty f(x) \dir x\right)\left(\int_0^\infty g(y) \dir y \right) = \int_0^\infty \int_0^\infty f(x)g(y) \dir x \dir y$.
And the last line applies the definition $\prod_{i=1}^d \left(\mathbbm{1}_{[c_i-r_i, c_i+r_i]}(x_i)\right) = \mathbbm{1}_{\prod_{i=1}^d[c_i-r_i,c_i+r_i]}(x)$.

Then we can apply this rewriting to the definition of kernel discrepancy for any $\chi$, and then factor out all of the positive terms.   What remains in the absolute value is exactly $\disc(P,\chi, R)$ for a rectangle $R$ with characteristic function $\mathbbm{1}_R(x) = \mathbbm{1}_{\prod_{i=1}^d[c_i-r_i,c_i+r_i]}(x)$, and hence is bounded by $\disc(P,\chi, \Eu{R}_d)$.  
	\begin{align*}
		\left| \sum_{p\in P} \chi(p) K(c,p) \right| 
		& =
		\left| \sum_{p\in P} \chi(p) \int_0^\infty\dots\int_0^\infty \prod_{i=1}^d\left( 2r_i\exp(-r_i^2)\right) \mathbbm{1}_{\prod_{i=1}^d[c_i-r_i,c_i+r_i]}(p) \dir r_1\dots \dir r_d \right| \\
		& =
		\left| \int_0^\infty\dots\int_0^\infty \prod_{i=1}^d\left( 2r_i\exp(-r_i^2)\right) \left(\sum_{p\in P} \chi(p) \mathbbm{1}_{\prod_{i=1}^d[c_i-r_i,c_i+r_i]}(p) \right) \dir r_1\dots \dir r_d \right| \\
		& \leq 
		\int_0^\infty\dots\int_0^\infty \prod_{i=1}^d\left( 2r_i\exp(-r_i^2)\right) \left| \sum_{p\in P} \chi(p) \mathbbm{1}_{\prod_{i=1}^d[c_i-r_i,c_i+r_i]}(p)\right| \dir r_1\dots \dir r_d  \\
		& \leq 
		\int_0^\infty\dots\int_0^\infty \prod_{i=1}^d\left( 2r_i\exp(-r_i^2)\right) \disc (P,\chi,\mathcal{R}_d) \dir r_1\dots \dir r_d  \\
		& =
		\disc(P,\chi,\mathcal{R}_d)\prod_{i=1}^d\left(\int_0^\infty 2r_i\exp(-r_i^2) \dir r_i\right) \\
		& \leq
		\disc(P,\chi,\mathcal{R}_d)
	\end{align*}
The last line follows by $\int_0^\infty 2r\exp(-r^2) \dir r = \exp(0) = 1$.   	
\end{proof}

Furthermore, we define 
\[
	\disc(n,\mathcal{R}_d) = \max_{\abs{P}=n}\min_{\chi} \disc(P,\chi,\mathcal{R}_d)
\qquad
\textrm{ and }
\qquad
	\disc(n,\mathcal{K}_d) = \max_{\abs{P}=n}\min_{\chi} \disc(P,\chi,\mathcal{K}_d).  
\]
Bansal and Garg~\cite{BG17} showed $\disc(n, \mathcal{R}_d) = O(\log^dn)$, and their proof provides a polynomial time algorithm.  
Nikolov~\cite{Nik17} soon after showed that $\disc(n,\mathcal{R}_d) = O(\log^{d-\frac{1}{2}}n)$ although this result does not describe how to efficiently construct the coloring.  
With these result we obtain the following.  

\begin{corollary}
$\disc (n,\Eu{K}_d)=O(\log^{d-\frac{1}{2}}n)$, and for any point set $P$ of size $n$, one can construct a coloring $\chi$ so that $\disc(P, \Eu{K}_d, \chi) = O(\log^d n)$.  
\end{corollary}

The following corollary is the direct implication of the above, for instance following~\cite{Phi13}.

\begin{corollary}
For any point set $P \subset \R^d$, there exists an $\eps$-kernel coreset of size 
$O(\frac{1}{\eps} \log^{d-\frac{1}{2}} \frac{1}{\eps})$.  
Moreover, an $\eps$-kernel coreset of size 
$O(\frac{1}{\eps} \log^{d} \frac{1}{\eps})$
can be constructed in $O(n + \poly(1/\eps))$ time
with high probability.  
\end{corollary}

\section{Lower Bound for Kernel Coresets}
\label{sec:LB}
In this section, we provide a lower bound matching or nearly matching our algorithms in Section \ref{sec:FW}.  To do so, we need to specify some properties of the kernels we consider.  We only consider shift invariant kernels, with a univariate function $f(|\|x-y\|) = K(x,y)$.  

Next we consider a class of functions we call \emph{somewhere-steep}.  For these kernels, there exists a region of $f$ where it is ``steep;'' its value consistently decreases quickly.  Specifically
\begin{itemize}
\item There exist constant $C_f>0$, and values $z_f>r_f>0$ such that $f(z_1)-f(z_2)>C_f\cdot(z_2-z_1)$ for all $z_1\in (z_f-r_f,z_f)$ and $z_2\in (z_f,z_f+r_f)$.  
\end{itemize}
Almost all kernels we have observed in literature (Gaussian, Laplace, Triangle, Epanechnikov, Sinc etc) are steep for all values $z_1, z_2 > 0$, and thus satisfy this property.  The exception is the ball kernel.  For this we define another class of kernels we call \emph{drop kernels}  where $f$ has a discontinuity where it drops more than a constant.  Specifically
\begin{itemize}
\item There exist constant $C_f>0$, and values $z_f>r_f>0$ such that $f(z_1)-f(z_2)>C_f$ for all $z_1\in (z_f-r_f,z_f)$ and $z_2\in (z_f,z_f+r_f)$.  
\end{itemize}

\paragraph{Construction.}
We now describe the construction used in the lower bounds; it is illustrated in Figure \ref{fig:LB}.  
Let $z_f$ be a scalar value that will depend on the specific kernel's univariate function $f$.  Now consider a point set $P=\{p_i = z_f e_i/\sqrt{2}\mid i=1,\dots,n\} \subset \mathbb{R}^n$ i.e. the scaled canonical basis, where $e_i = (0,0,\ldots,0,1,0,\ldots,0)$ with the $1$ in the $i$th coordinate. 
We can select $Q=\{p_i\mid i=1,\dots,k\}$, without loss of generality, because of the symmetry in $P$.  
Denote $\kde_Q=\sum_{i=1}^k\beta_i\phi_i$ for some $\beta_1, \beta_2, \ldots$ where $\sum_{i=1}^k\beta_i=1$. 
Let 
$\bar{p}=\frac{1}{n}\sum_{i=1}^n p_i$ 
be the mean of points in $P$, 
$\bar{p}_k=\frac{1}{k}\sum_{i=1}^k p_i$ 
be the mean of points in $Q$.  
Let 
$\bar{p}_{-k}=\frac{1}{n-k}\sum_{i=k+1}^n p_i$
be mean of points in $P\setminus Q$, and 
$p=\bar{p}_k+\frac{z_f}{\sqrt{2}}\frac{\bar{p}_k-\bar{p}}{\norm{\bar{p}_k-\bar{p}}}$ 
be the point lying on the line $\bar{p}\bar{p}_k$ such that $p$ and $\bar{p}$ are on the opposite side of $\bar{p}_k$ and $\norm{p-\bar{p}_k}=\frac{z_f}{\sqrt{2}}$.  
		Note that, for all $i=1,\dots,k$, $\norm{p-p_i}$ are the same which is denoted $l_1$ and, 
		for all $i=k+1,\dots,n$, $\norm{p-p_i}$ are the same which is denoted $l_2$. 

\begin{figure}
\begin{center}
	\includegraphics[width=0.8\textwidth]{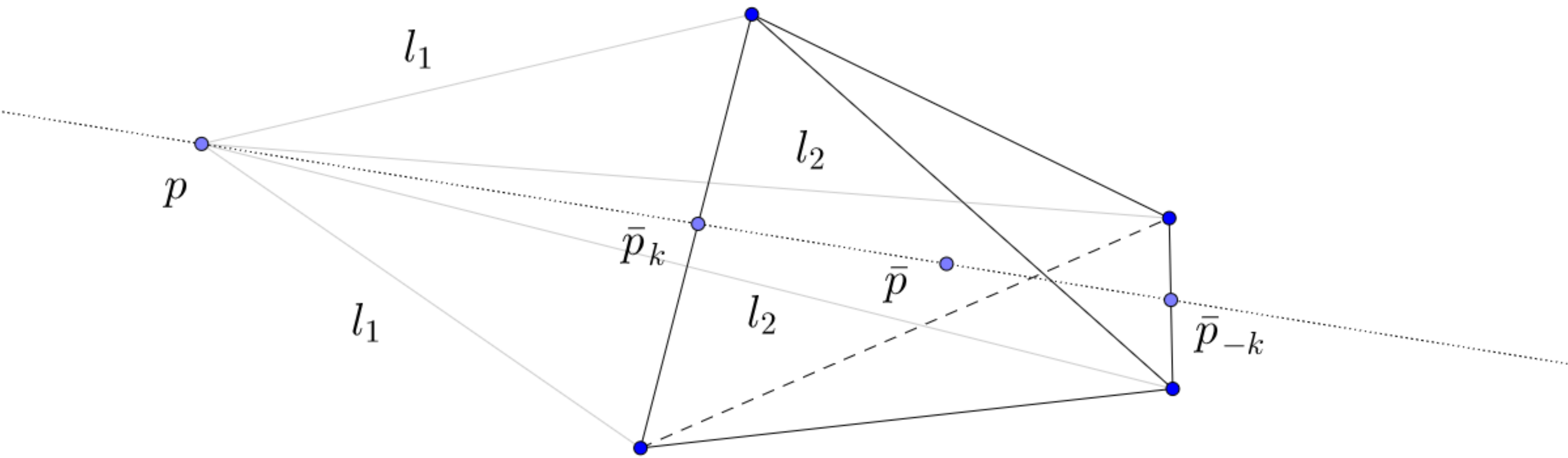}
\end{center}
\vspace{-.25in}
	\caption{Illustration of the lower bound construction.\label{fig:LB}}
\end{figure}

Now we evaluate $\kde_P-\kde_Q$ at $p$, resulting in
\begin{align*}
(\kde_Q-\kde_P)(p)
& =
\sum_{i=1}^k(\beta_i-\frac{1}{n})f(\norm{p-p_i})+\sum_{i=k+1}^n(-\frac{1}{n})f(\norm{p-p_i}) 
\\& =
\sum_{i=1}^k(\beta_i-\frac{1}{n})f(l_1)+\sum_{i=k+1}^n(-\frac{1}{n})f(l_2)
\\& =
(\sum_{i=1}^k\beta_i-\frac{k}{n})f(l_1)+(n-k)(-\frac{1}{n})f(l_2) 
\\& =
(1-\frac{k}{n})(f(l_1)-f(l_2)).  \numberthis \label{kdeerr}  
\end{align*}

\begin{lemma} \label{lem:l1l2}
$l_1^2 =z_f^2 - z_f^2 / (2k)$ and 
\[
l_2^2 = \frac{z_f^2}{2} \left(1 + \sqrt{\frac{1}{k} - \frac{1}{n}} + \sqrt{\frac{1}{n-k} - \frac{1}{n}}\right)^2 + \frac{z_f^2}{2} \left(1-\frac{1}{n-k}\right).  
\]
\end{lemma}
\begin{proof}
Since $\bar{p}=(\frac{k}{n})\bar{p}_k+(\frac{n-k}{n})\bar{p}_{-k}$, it means that $\bar{p}_{-k}$ also lies on line $\bar{p}\bar{p}_k$. 
Moreover, $\langle\bar{p}_k-\bar{p},\bar{p}_{-k}-\bar{p}\rangle=-\frac{k}{n-k}\norm{\bar{p}-\bar{p}_k}^2\leq 0$ which shows that $\bar{p}_k$ and $\bar{p}_{-k}$ lies on different side of $\bar{p}$. For any $i=1,\dots,k$, $\langle \bar{p}-\bar{p}_k, \bar{p}_k-p_i\rangle = 0$ which means that the line $\bar{p}\bar{p}_k$ is perpendicular to the subspace of the linear combination of $\{p_1,\dots,p_k\}$.  A similar argument can be applied for $\bar{p}_{-k}$. 
Also, note that, for all $i=1, \dots,k$, $\norm{\bar{p}_k-p_i}$ are the same and are equal to $\frac{z_f}{\sqrt{2}}\sqrt{1-\frac{1}{k}}$, and for all $i=k+1,\dots,n$, $\norm{\bar{p}_{-k}-p_i}$ are same and equal to $\frac{z_f}{\sqrt{2}}\sqrt{1-\frac{1}{n-k}}$. 
Moreover, it is easy to compute $\norm{\bar{p}-\bar{p}_k}=\frac{z_f}{\sqrt{2}}\sqrt{\frac{1}{k}-\frac{1}{n}}$ and $\norm{\bar{p}-\bar{p}_{-k}}=\frac{z_f}{\sqrt{2}}\sqrt{\frac{1}{n-k}-\frac{1}{n}}$.

For all $i=1,\dots,k$,
\[
l_1^2 = \norm{p-\bar{p}_k}^2+\norm{\bar{p}_k-p_i}^2= \frac{z_f^2}{2}+\frac{z_f^2}{2}(1-\frac{1}{k})=z_f^2-\frac{z_f^2}{2k}.
\]

For all $i=k+1,\dots,n$,
\[
l_2^2 
=	
\norm{p-\bar{p}_{-k}}^2+\norm{\bar{p}_{-k}-p_i}^2
=
\frac{z_f^2}{2} \left( 1+\sqrt{\frac{1}{k}-\frac{1}{n}}+\sqrt{\frac{1}{n-k}-\frac{1}{n}}\right)^2+\frac{z_f^2}{2}\left(1-\frac{1}{n-k}\right).  \qedhere
\]
\end{proof}

\paragraph{Analysis.}
Since the choice of $Q$ is arbitrary due to the symmetry of $P$, if we can show (\ref{kdeerr}) is sufficiently large as a function of $k$, then we can prove a lower bound.  Given a careful choice of $z_f$ the depends on the kernel, we now evaluate (\ref{kdeerr}) with respect to $k$ using the definitions of $l_1$ and $l_2$ specified in Lemma \ref{lem:l1l2}.
But first we observe that a (very minor)%
\footnote{  
The bounds will contain some complicated appearing terms depending on $z_f$ and $r_f$.  We can set $z_f/2 = r_f$, then we observe they are small constants
$
(\frac{4 z_f^2}{2 z_f r_f  + r_f^2})^2 
= 
(\frac{4 z_f^2}{z_f^2  + 0.25z_f^2})^2 
=
(\frac{4}{1.25})^2
=10.24, 
$
and 
$
\frac{z_f^2}{2(2 z_f r_f - r_f^2)}
=
\frac{z_f^2}{2z_f^2 - 0.5z_f^2}
\approx
0.666.
$
}
restriction on $k$, we can ensure that $l_1$ and $l_2$ are in the proper interval with respect to $z_f$ and $r_f$.  

\begin{lemma} \label{lem:interval}
If $\min\{n-1, n-(\frac{4z_f^2}{2z_fr_f+r_f^2})^2\} \geq k\geq \max\{\frac{z_f^2}{2(2z_fr_f-r_f^2)}, (\frac{4z_f^2}{2z_fr_f+r_f^2})^2, 1\}$ then 
$l_1\in (z_f-r_f,z_f)$ and $l_2\in(z_f,z_f+r_f)$. 
\end{lemma}
\begin{proof}
Clearly from Lemma \ref{lem:l1l2}, $l_1^2$ is less than $z_f^2$. Also,
\[
\frac{z_f^2}{2k}\leq \frac{z_f^2}{2}/\left(\frac{z_f^2}{2(2z_fr_f-r_f^2)}\right)=2z_fr_f-r_f^2
\]
which means $l_1^2>z_f^2-(2z_fr_f-r_f^2)=(z_f-r_f)^2$. 
		
Again using Lemma \ref{lem:l1l2} we see
		\begin{align*}
			l_1^2& \geq 
			\frac{z_f^2}{2}(1+\sqrt{\frac{1}{n-k}-\frac{1}{n}})^2+\frac{z_f^2}{2}(1-\frac{1}{n-k})\\
			& =
			\frac{z_f^2}{2}(1+2\sqrt{\frac{1}{n-k}-\frac{1}{n}}+\frac{1}{n-k}-\frac{1}{n}+1-\frac{1}{n-k}) \\
			& =
			\frac{z_f^2}{2}(2+2\sqrt{\frac{1}{n-k}-\frac{1}{n}}-\frac{1}{n}) 
		 \geq
			z_f^2. 
		\end{align*}
Moreover from Lemma \ref{lem:l1l2}, 
		\begin{align*}
			l_2^2 & \leq 
			\frac{z_f^2}{2}(1+2\max\{\sqrt{\frac{1}{k}},\sqrt{\frac{1}{n-k}}\})^2+\frac{z_f^2}{2} \\
			& =
			z_f^2+2z_f^2\max\{\sqrt{\frac{1}{k}},\sqrt{\frac{1}{n-k}}\}+2z_f^2\max\{\sqrt{\frac{1}{k}},\sqrt{\frac{1}{n-k}}\}^2 \\
			& \leq
			 z_f^2+4z_f^2\max\{\sqrt{\frac{1}{k}},\sqrt{\frac{1}{n-k}}\} \\
			& \leq
			z_f^2+4z_f^2/\left(\frac{4z_f^2}{2z_fr_f+r_f^2}\right) 
			 =
			(z_f+r_f)^2.  
		\end{align*}
That is $l_1\in (z_f-r_f,z_f)$ and $l_2\in(z_f,z_f+r_f)$.
\end{proof}

\begin{lemma} \label{lb1}
Consider a shift-invariant drop kernel $K$.  
There exists a set $P$ of size $n$ so any subset $Q \subset P$ such that $\|\kde_P - \kde_Q\|_\infty \leq \eps$ requires that $k \geq n - O(\eps n)$, 
 assuming $\min\{n-1, n-(\frac{4z_f^2}{2z_fr_f+r_f^2})^2\} \geq k\geq \max\{\frac{z_f^2}{2(2z_fr_f-r_f^2)}, (\frac{4z_f^2}{2z_fr_f+r_f^2})^2, 1\}$.
\end{lemma}

\begin{proof}
From Lemma \ref{lem:interval} we have $l_1\in (z_f-r_f,z_f)$ and $l_2\in(z_f,z_f+r_f)$, thus by the definition of a drop kernel and (\ref{kdeerr}) we observe
\[
(\kde_Q-\kde_P)(p)
 =
(1-\frac{k}{n})(f(l_1)-f(l_2)) 
\geq 
(1-\frac{k}{n})C_f.  
\]
Thus if $\norm{\kde_Q-\kde_P}_\infty\leq \eps$, then $k\geq n-O(\eps n)$.	
\end{proof}

\begin{lemma}
Consider a shift-invariant somewhere-steep $K$.  
There exists a set $P$ of size $n$ so any subset $Q \subset P$ such that $\|\kde_P - \kde_Q\|_\infty \leq \eps$ requires that $k = \Omega(1/\eps^2)$, 
assuming $n/2 \geq k\geq \max\{\frac{z_f^2}{2(2z_fr_f-r_f^2)}, (\frac{4z_f^2}{2z_fr_f+r_f^2})^2, 1\}$.
\end{lemma}

\begin{proof}	
We can write $l_2^2$ in terms of $l_1^2$.  
	\begin{align*}
		l_2^2
		& =
		\norm{p-\bar{p}_{-k}}^2+\norm{\bar{p}_{-k}-p_i}^2
		\;\;\; \geq \;\;\; 
		\norm{p-\bar{p}}^2+\norm{\bar{p}_{-k}-p_i}^2\\
		& =
		\frac{z_f^2}{2}\left(1+\sqrt{\frac{1}{k}-\frac{1}{n}}\right)^2+\frac{z_f^2}{2}(1-\frac{1}{n-k}) \\
		& \geq
		\frac{z_f^2}{2}\left(1+\sqrt{\frac{1}{k}-\frac{1}{n}}\right)^2+\frac{z_f^2}{2}(1-\frac{1}{k})
		\;\;=\;\;
		\frac{z_f^2}{2}\left(1+2\sqrt{\frac{1}{k}-\frac{1}{n}}+(\frac{1}{k}-\frac{1}{n})+(1-\frac{1}{k})\right)\\
		& \geq
		\frac{z_f^2}{2}\left(1+2\sqrt{\frac{1}{k}-\frac{1}{n}}+(1-\frac{1}{k})\right)\\
		& =
		l_1^2+z_f^2\sqrt{(1-\frac{k}{n})}\sqrt{\frac{1}{k}}  
		\;\;\; \geq \;\;\;
		l_1^2+z_f^2\sqrt{\frac{1}{2k}}.  
	\end{align*}

From Lemma \ref{lem:interval} we have $l_1\in (z_f-r_f,z_f)$ and $l_2\in(z_f,z_f+r_f)$. So from the definition of a somewhere-steep kernel, we can conclude from (\ref{kdeerr})  and $n/2 \geq k$ that
	\begin{align*}
		(\kde_Q-\kde_P)(p)
		& =
		(1-\frac{k}{n})(f(l_1)-f(l_2)) 
	    \geq 
		\frac{1}{2}(f(l_1)-f(l_2)) 
		 =
		C_f(l_2-l_1)\\
		& = 
		C_f\left(\sqrt{l_1^2+z_f^2\sqrt{\frac{1}{2k}}}-l_1\right) 
		 =
		C_f\left(\sqrt{l_1^2+z_f^2\sqrt{\frac{1}{2k}}}+l_1\right)^{-1}z_f^2\sqrt{\frac{1}{2k}} \\
		& \geq 
		\frac{C_fz_f}{3}\sqrt{\frac{1}{2k}}.  
	\end{align*}
	If $\norm{\kde_Q-\kde_P}_\infty\leq \eps$, then $k=\Omega(1/\eps^2)$.
\end{proof}

\begin{theorem}
For the Gaussian or Laplace kernel, there is a set $P$ so for any subset $Q \subset P$ such that $\|\kde_P - \kde_Q\|_\infty \leq \eps$, then $|Q|=\Omega(1/\eps^2)$.  
\end{theorem}

\bibliographystyle{plain}
\bibliography{discrepancy}

\begin{thebibliography}{10}

\bibitem{aronszajn1950theory}
Nachman Aronszajn.
\newblock Theory of reproducing kernels.
\newblock {\em Transactions of the American mathematical society},
  68(3):337--404, 1950.

\bibitem{BLO12}
Franics Bach, Simon Lacsote-Julien, and Guillaume Obozinski.
\newblock On the equivalence between herding and conditional gradient
  algorithms.
\newblock In {\em Proceedings International Conference on Machine Learning},
  2012.

\bibitem{BG17}
Nikhil Bansal and Shashwat Garg.
\newblock Algorithmic discrepancy beyond partial coloring.
\newblock In {\em Proceedings ACM Symposium on Theory of Computation}, 2017.

\bibitem{CWS10}
Yutian Chen, Max Welling, and Alex Smola.
\newblock Super-samples from kernel hearding.
\newblock In {\em Conference on Uncertainty in Artificial Intellegence}, 2010.

\bibitem{Cla10}
Ken Clarkson.
\newblock Coresets, sparse greedy approximation, and the frank-wolfe algorithm.
\newblock {\em ACM Transactions on Algorithms}, 4(6), 2010.

\bibitem{CS15}
Efren~Cruz Cortes and Clayton Scott.
\newblock Sparse approximation of a kernel mean.
\newblock {\em IEEE Transactions on Signal Processing}, accepted
  (arXiv:1503.00323), 2015.

\bibitem{DG84}
Luc Devroye and L\'{a}szl\'{o} Gy\"{o}rfi.
\newblock {\em Nonparametric Density Estimation: The $L_1$ View}.
\newblock Wiley, 1984.

\bibitem{dunn1980convergence}
Joseph~C Dunn.
\newblock Convergence rates for conditional gradient sequences generated by
  implicit step length rules.
\newblock {\em SIAM Journal on Control and Optimization}, 18(5):473--487, 1980.

\bibitem{FLRWBS14}
Brittany~Terese Fasy, Fabrizio Lecci, Alessandro Rinaldo, Larry Wasserman,
  Sivaraman Balakrishnan, and Aarti Singh.
\newblock Confidence sets for persistence diagrams.
\newblock {\em The Annals of Statistics}, 42:2301--2339, 2014.

\bibitem{FW56}
Marguerite Frank and Philip Wolfe.
\newblock An algorithm for quadratic programming.
\newblock {\em Naval Research Logistics Quarterly}, 3:95--110, 1956.

\bibitem{FG16}
Robert Freund and Paul Grigas.
\newblock New analysis and results for the frank-wolfe method.
\newblock {\em Mathematical Programming}, (to appear).

\bibitem{GJ09}
Bernd G\"{a}rtner and Martin Jaggi.
\newblock Coresets for polytope distance.
\newblock In {\em SOCG}, 2009.

\bibitem{glaunesthesis}
Joan Glaun\`{e}s.
\newblock {\em Transport par diff\'{e}omorphismes de points, de mesures et de
  courants pour la comparaison de formes et l'anatomie num\'{e}rique.}
\newblock PhD thesis, Universit\'{e} Paris 13, 2005.

\bibitem{Gon85}
Teofilo~F. Gonzalez.
\newblock Clustering to minimize the maximum intercluster distance.
\newblock {\em Theoretical Computer Science}, 38:293--306, 1985.

\bibitem{GBRSS12}
Arthur Gretton, Karsten~M. Borgwardt, Malte~J. Rasch, Bernhard Scholkopf, and
  Alexander Smola.
\newblock A kernel two-sample test.
\newblock {\em JMLR}, 13:723--773, 2012.

\bibitem{HBCM13}
Zaid Harchaoui, Francis Bach, Olivier Cappe, and Eric Moulines.
\newblock Kernel-based methods for hypothesis testing: A unified view.
\newblock {\em IEEE Signal Processing Magazine}, 30:87--97, 2013.

\bibitem{harvey2014near}
Nick Harvey and Samira Samadi.
\newblock Near-optimal herding.
\newblock In {\em COLT}, pages 1165--1182, 2014.

\bibitem{HB05}
Matrial Hein and Olivier Bousquet.
\newblock Hilbertian metrics and positive definite kernels on probability
  measures.
\newblock In {\em AIStats}, 2005.

\bibitem{Jag13}
Martin Jaggi.
\newblock Revisiting {F}rank-{W}olfe: Projection-free sparse convex
  optimization.
\newblock In {\em ICML}, 2013.

\bibitem{JL15}
Martin Jaggi and Simon Lacsote-Julien.
\newblock On the global linear convergence of frank-wolfe optimization
  variants.
\newblock In {\em NIPS}, 2015.

\bibitem{JoshiKommarajuPhillips2011}
Sarang Joshi, Raj~Varma Kommaraju, Jeff~M. Phillips, and Suresh
  Venkatasubramanian.
\newblock Comparing distributions and shapes using the kernel distance.
\newblock In {\em SOCG}, 2011.

\bibitem{JKPV11}
Sarang Joshi, Raj~Varma Kommaraju, Jeff~M. Phillips, and Suresh
  Venkatasubramanian.
\newblock Comparing distributions and shapes using the kernel distance.
\newblock In {\em 27th Annual Symposium on Computational Geometry}, 2011.

\bibitem{LLS01}
Yi~Li, Philip~M. Long, and Aravind Srinivasan.
\newblock Improved bounds on the samples complexity of learning.
\newblock {\em J. Comp. and Sys. Sci.}, 62:516--527, 2001.

\bibitem{muller1997integral}
Alfred M{\"u}ller.
\newblock Integral probability metrics and their generating classes of
  functions.
\newblock {\em Advances in Applied Probability}, 29(2):429--443, 1997.

\bibitem{Nik17}
Aleksandar Nikolov.
\newblock Tighter bounds for the discrepancy of boxes and polytopes.
\newblock Technical report, arXiv:1701.05532, 2017.

\bibitem{Par62}
Emanuel Parzen.
\newblock On the estimation of a probability density function and the mode.
\newblock {\em Annals of Mathematical Statistics}, 33:1065--1076, 1962.

\bibitem{Phi13}
Jeff~M. Phillips.
\newblock $\eps$-samples for kernels.
\newblock In {\em SODA}, 2013.

\bibitem{PhillipsVenkatasubramanian2011}
Jeff~M. Phillips and Suresh Venkatasubramanian.
\newblock A gentle introduction to the kernel distance.
\newblock arXiv:1103.1625, March 2011.

\bibitem{PWZ15}
Jeff~M. Phillips, Bei Wang, and Yan Zheng.
\newblock Geometric inference on kernel density estimates.
\newblock In {\em SOCG}, 2015.

\bibitem{SS02}
Bernhard Scholkopf and Alexander~J. Smola.
\newblock {\em Learning with Kernels: Support Vector Machines, Regularization,
  Optimization, and Beyond}.
\newblock MIT Press, 2002.

\bibitem{Sco92}
David~W. Scott.
\newblock {\em Multivariate Density Estimation: Theory, Practice, and
  Visualization}.
\newblock Wiley, 1992.

\bibitem{FWS15}
Alexander J.~Smola Seth R.~Flaxman, Yu-Xiang~Wang.
\newblock Who supported {O}bama in 2012?: Ecological inference through
  distribution regression.
\newblock In {\em KDD}, 2015.

\bibitem{Sil86}
Bernard~W. Silverman.
\newblock {\em Density Estimation for Statistics and Data Analysis}.
\newblock Chapman \& Hall/CRC, 1986.

\bibitem{SZSGS08}
Le~Song, Xinhua Zhang, Alex Smola, Arthur Gretton, and Berhard Sch\"olkopf.
\newblock Tailoring density estimation via reproducing kernel moment matching.
\newblock In {\em ICML}, 2008.

\bibitem{SGFSL10}
Bharath~K. Sriperumbudur, Arthur Gretton, Kenji Fukumizu, Bernhard Sch\"olkopf,
  and Gert R.~G. Lanckriet.
\newblock Hilbert space embeddings and metrics on probability measures.
\newblock {\em JMLR}, 11:1517--1561, 2010.

\bibitem{Wah99}
Grace Wahba.
\newblock Support vector machines, reproducing kernel {H}ilbert spaces, and
  randomization.
\newblock In {\em Advances in Kernel Methods -- Support Vector Learning}, pages
  69--88. 1999.

\bibitem{ZP15}
Yan Zheng and Jeff~M. Phillips.
\newblock L\_infity error and bandwidth selection for kernel density estimates
  of large data.
\newblock In {\em KDD}, 2015.

\end{thebibliography}

\end{document}